%% file: main.tex
\documentclass{article}

\usepackage[preprint]{neurips}

\usepackage[utf8]{inputenc} 
\usepackage[T1]{fontenc}    
\usepackage{hyperref}       
\usepackage{url}            
\usepackage{booktabs}       
\usepackage{amsfonts}       
\usepackage{nicefrac}       
\usepackage{microtype}      
\usepackage{xcolor}         

\usepackage{amssymb}            
\usepackage{mathtools}          
\usepackage{mathrsfs}           
\usepackage{graphicx}           
\usepackage[space]{grffile}     
\usepackage{url}                
\usepackage{algorithm}      
\usepackage{algorithmic}    
\usepackage{amsmath}        
\usepackage{amsthm}
\usepackage{booktabs, caption, subcaption}

\theoremstyle{plain}
\newtheorem{theorem}{Theorem}
\newtheorem{lemma}{Lemma}

\newtheorem{corollary}{Corollary}

\theoremstyle{definition}

\newtheorem{assumption}{Assumption}

\theoremstyle{remark}
\newtheorem{remark}{Remark}
\newcommand{\defeq}{\coloneqq}
\title{FedQHD: Closed-Form Function-Space Federated Reinforcement Learning}

\author{%
  \textbf{Yuchen Hou}\textsuperscript{1} \quad
  \textbf{Yongshan Chen}\textsuperscript{1} \quad
  \textbf{Zhuowen Zou}\textsuperscript{2} \quad
  \textbf{Calvin Yeung}\textsuperscript{2} \\
  \textbf{Mohsen Imani}\textsuperscript{2} \quad
  \textbf{Tian Lan}\textsuperscript{3} \quad
  \textbf{Mahdi Imani}\textsuperscript{1} \\[0.4em]
  \normalsize
  \textsuperscript{1}Northeastern University \quad
  \textsuperscript{2}University of California, Irvine \quad
  \textsuperscript{3}The George Washington University \\
  \texttt{\{hou.yuchen, chen.yongs, m.imani\}@northeastern.edu} \\
  \texttt{\{zhuowez1, chyeung2, m.imani\}@uci.edu} \quad
  \texttt{tlan@gwu.edu}
}

\begin{document}

\maketitle

\begin{abstract}
  Federated reinforcement learning enables decentralized agents to collaboratively improve policies or value estimates without exchanging raw trajectories. However, FedAvg-style parameter averaging is not function-space consistent: when clients use heterogeneous encoders or even identical nonlinear networks, averaged parameters need not correspond to the weighted average of client value functions in any common function space.
We propose \emph{FedQHD}, a federated Q-learning method using hyperdimensional (random-feature) state encoders with a linear readout, so that Q-functions are nonlinear in state yet linear in trainable parameters. This linear structure enables closed-form aggregation. With a shared encoder, the function-space consensus update coincides exactly with weighted averaging of local readout matrices. With heterogeneous encoders, the server constructs a global teacher by averaging client Q-values on a shared anchor-state set, and each client compiles this teacher into its local representation via a single ridge projection. We formalize the \emph{federation gap}---the error incurred when compiling a federated teacher into a heterogeneous client representation---relative to a client-specific oracle projection. We show that this gap decomposes into subspace misalignment, anchor-set conditioning, and regularization bias. We further identify the anchor-to-dimension ratio $m \geq D_i$ as the well-conditioned regime in which the gap reduces to a multiple of the encoder heterogeneity floor. On four continuous-state, discrete-action control benchmarks, FedQHD matches or outperforms FedAvg-style baselines and distillation-based alternatives while requiring substantially less computation, and the empirical
dependence of the federation gap on encoder dimension matches our theoretical analysis.
\end{abstract}

\input{sections/intro}
\input{sections/related_work}
\input{sections/preliminaries}

\input{sections/methodology}

\input{sections/experiments}
\input{sections/conclusion}
\bibliography{main}
\bibliographystyle{plainnat}

\appendix
\input{appendix/exp}
\input{appendix/projection_residual_invisibility}
\input{appendix/encoder_hetero}
\input{appendix/regression_error}

\end{document}

%% file: sections/intro.tex
\section{Introduction}
\label{sec:intro}
Reinforcement learning (RL) systems in autonomous vehicles~\citep{liang2022federated, chellapandi2023federated}, industrial robots~\citep{liu2019lifelong}, and resource-constrained edge devices~\citep{yu2020deep} often learn from on-device interaction data that cannot be centralized due to communication costs, privacy requirements, and the volume of on-device experience. Federated reinforcement learning (FedRL) targets this setting by allowing agents to improve jointly without sharing raw trajectories~\citep{zhuo2020federated, qi2021federated}.


Most FedRL pipelines inherit parameter averaging (FedAvg) from supervised federated learning~\citep{mcmahan2017communication}: clients train locally, the server averages parameters, and the averaged model is broadcast back. However, federated \emph{Q}-learning exposes two structural obstacles. First, weight averaging of nonlinear value networks is not value-function averaging; achieving function-space agreement typically requires additional optimization. Second, practical deployments are structurally heterogeneous: clients may use different encoders, feature dimensions, or architectures, making parameter averaging algebraically undefined~\citep{fan2023fedhql, jiang2025fedhpd}.


The dominant approach to heterogeneous federation is knowledge distillation~\citep{li2019fedmd, lin2021ensembledistillation, jiang2025fedhpd}, which exchanges predictions on shared query states and iteratively trains local students toward an ensembled teacher. Distillation introduces per-round iterative optimization, hyperparameter sensitivity, and instability under the nonstationary Bellman targets of online RL~\citep{czarnecki2019distilling}. We pursue an alternative that remains well-defined under heterogeneous representations without iterative teacher--student training.

Hyperdimensional computing (HDC), and more broadly fixed random-feature value approximation, offers an alternative value representation: states are mapped through a fixed high-dimensional feature map and action values are produced by a linear readout~\citep{kanerva2009hyperdimensional}. This linear readout enables closed-form least-squares-style updates and avoids backpropagation in hyperdimensional Q-learning (QHD)~\citep{ni2022qhd}. The linearity in the trainable parameters also simplifies federation: for linear-in-parameters value functions, averaging in value function space coincides exactly with averaging parameters~\citep{lagoudakis2003least, bhandari2018finite}, and heterogeneous aggregation reduces to a projection step rather than iterative distillation.

We propose \emph{FedQHD}, a federated $Q$-learning framework that aggregates clients through their $Q$-values and remains well-defined under encoder heterogeneity. With a shared encoder, the federated update reduces exactly to weighted averaging of readout matrices, recovering FedAvg in closed form. With heterogeneous encoders, the server forms a teacher by averaging client $Q$-values on a shared \emph{anchor} set, and each client compiles this teacher into its own representation via a single ridge-regression solve per round---without exchanging trajectories and without iterative optimization.
 
Our contributions are:
\begin{itemize}
    \item \textbf{Closed-form federation under heterogeneous encoders.} We propose a closed-form federated $Q$-learning algorithm that handles heterogeneous encoders in a single step, compiling a function-space teacher into each client's local representation via anchor-based ridge regression and recovering FedAvg exactly when encoders are shared.
    \item \textbf{Pointwise bound on the federation gap.} We derive a pointwise bound that decomposes the gap into three interpretable terms---encoder heterogeneity, anchor conditioning, and ridge shrinkage---and identify $m \geq D_i$ as the well-conditioned regime in which the gap reduces to a multiple of the heterogeneity floor.
    \item \textbf{Empirical validation on four continuous-control benchmarks.} We conduct experiments on four continuous-control tasks under both homogeneous and heterogeneous encoders, showing that FedQHD matches or exceeds federated DQN baselines while running substantially faster than distillation-based alternatives, with ablations confirming the predicted dependence on encoder dimension and anchor-set size.
\end{itemize}

%% file: sections/related_work.tex
\section{Related Work}
\label{sec:related_work}

\paragraph{Federated RL with shared parameterizations.}
Federated learning was popularized by FedAvg, which aggregates client models through iterative parameter averaging~\citep{mcmahan2017communication}. Early federated RL systems applied this paradigm by sharing neural value or policy network parameters across agents~\citep{zhuo2020federated,nadiger2019federated}, including applications such as autonomous driving under distribution shift~\citep{liang2022federated} and Byzantine-robust policy gradients~\citep{fan2021fault}. More recent work established finite-time guarantees for federated TD and Q-learning under Markovian sampling~\citep{khodadadian2022federated} and analyzed performance degradation under environment heterogeneity~\citep{jin2022federated}. However, these approaches assume a \emph{shared parameterization} across clients: FedAvg-style aggregation requires identical parameter shapes and is undefined when clients use different encoders or feature dimensions.

\paragraph{Variance reduction and personalization.}
Several works address optimization drift in FedAvg. FedProx~\citep{li2020federated} introduces proximal regularization, while SCAFFOLD~\citep{karimireddy2020scaffold} uses control variates to reduce client variance. Personalized federated learning methods further allow each client to maintain a locally adapted model~\citep{fallah2020personalized}. In contrast, FedQHD eliminates client drift entirely in the homogeneous case (exact aggregation) and handles heterogeneous encoders through a closed-form ridge projection rather than iterative optimization.

\paragraph{Distillation-based federation under heterogeneity.}
Knowledge distillation aggregates models in output space rather than parameter space~\citep{hinton2015distilling}, enabling federation across heterogeneous architectures. In supervised federated learning~\citep{li2019fedmd,lin2021ensembledistillation,zhu2021data,chen2020fedbe}, methods differ in proxy-data assumptions but all rely on iterative gradient-based fitting. In RL, policy distillation~\citep{rusu2016policydistillation} and Distral~\citep{teh2017distral} introduced function-space transfer mechanisms. Recent heterogeneous FedRL approaches adopt similar principles: FedHQL aggregates models through server-side queries~\citep{fan2023fedhql}, SCCD distills ensembles using pseudo-data~\citep{mai2023server}, and FedHPD matches action distributions on shared anchor states~\citep{jiang2025fedhpd}. These approaches require iterative teacher--student optimization and can be sensitive to design and hyperparameter choices, particularly under nonstationary Bellman targets~\citep{czarnecki2019distilling}.

\paragraph{Linear function approximation, kernels, and random features in RL.}
Linear value-function approximation has long provided stable and analyzable RL algorithms. Least-squares approaches such as LSPI and fitted Q-iteration formulate Bellman updates as regression problems with closed-form solutions~\citep{lagoudakis2003least,ernst2005tree}, while finite-time guarantees for linear TD have been established under both i.i.d.\ and Markovian sampling~\citep{bhandari2018finite}. Kernel and basis-function methods extend this framework to nonlinear state representations while retaining linear parameter structure~\citep{ormoneit2002kernel,konidaris2011value}. Random Fourier features provide scalable kernel approximations~\citep{rahimi2007random}, and regret analyses connect reproducing kernel Hilbert space (RKHS) geometry to RL sample complexity~\citep{jin2020provably}. HDC~\citep{kanerva2009hyperdimensional} can be viewed as a high-dimensional random-feature instantiation; QHD and HDPG demonstrate that HDC encoders enable efficient RL with linear readouts and least-squares-style updates~\citep{ni2022qhd,ni2022hdpg}.

\paragraph{Positioning of FedQHD.}
FedQHD addresses \emph{structural heterogeneity} in federated Q-learning, where clients may use different encoders and parameter dimensions and parameter averaging becomes ill-defined~\citep{fan2023fedhql,jiang2025fedhpd}. Instead of iterative distillation, FedQHD aggregates Q-values on a shared anchor-state interface and compiles the resulting consensus into each client representation via a one-shot ridge projection. In the homogeneous limit (shared encoder), this procedure reduces exactly to parameter averaging, connecting classical federated learning with heterogeneous value-function aggregation.

%% file: sections/preliminaries.tex
\section{Preliminaries}
\label{sec:preliminaries}

\subsection{Markov decision processes and off-policy value learning}
We consider a Markov decision process $\mathcal{M} = (\mathcal{S}, \mathcal{A}, P, r, \gamma)$, where $\mathcal{S}$ is a continuous state space, $\mathcal{A}$ is a finite action set, $P:\mathcal{S}\times\mathcal{A}\!\to\!\Delta(\mathcal{S})$ is the transition kernel, $r:\mathcal{S}\times\mathcal{A}\!\to\!\mathbb{R}$ is the reward function, and $\gamma\in(0,1)$ is the discount factor. The optimal action-value function $Q^\star$ is the unique fixed point of the Bellman optimality operator
\[
Q^\star(s,a)= r(s,a) + \gamma\,\mathbb{E}_{s'\sim P(\cdot|s,a)}
\left[\max_{a'\in\mathcal{A}} Q^\star(s',a')\right],
\]
with optimal policy $\pi^\star(s)=\arg\max_a Q^\star(s,a)$. Since $\mathcal{S}$ is continuous, $Q^\star$ is approximated using standard off-policy temporal-difference framework: transitions $(s,a,r,s')$ are stored in a replay buffer and $Q^\star$ is estimated by semi-gradient updates against a periodically frozen target network~\citep{ni2022qhd}.



\subsection{Hyperdimensional computing (HDC)}

\label{subsec:hdc_background}

HDC is a brain-inspired computational paradigm in which symbols and structured entities are represented as high-dimensional vectors---called \emph{hypervectors}---with components drawn independently from simple distributions~\citep{kanerva2009hyperdimensional}. In such spaces, independently sampled hypervectors are nearly orthogonal with high probability, a geometric property that underlies classical HDC operations such as \emph{bundling} (superposition via addition), \emph{binding} (association via elementwise multiplication or permutation), and \emph{permutation} (role shifting). Because information is distributed holographically across all dimensions, HDC representations exhibit strong robustness to noise, quantization, and partial corruption.

Given an input $x$, an HDC encoder produces a bounded hypervector $\phi(x)\in\mathbb{K}^D$ using random projections or compositional schemes built from a small set of base hypervectors. These encoders admit a precise random-feature interpretation: the empirical kernel $k_D(x,x')=\langle \phi(x),\phi(x')\rangle$ 
converges to a smooth limiting kernel $k_\ast(x,x')$ as $D\to\infty$, with uniform concentration at rate $O(D^{-1/2})$ \citep{rahimi2007random,bach2015equivalence,rudi2017generalization}. As a consequence, linear prediction in the hypervector domain serves as a computationally efficient, finite-dimensional approximation to kernel methods in the RKHS associated with $k_\ast$.





%% file: sections/methodology.tex
\section{Problem Formulation}
\label{sec:problem}

\paragraph{QHD: hyperdimensional Q-learning for continuous-state control.}
We consider an MDP for each agent $i$ with continuous states $s\in\mathcal{S}$ and discrete actions $a\in\mathcal{A}$.  Agent $i$ selects a hyperdimensional encoder $
\Phi_i:\mathcal{S}\to\mathbb{K}^{D_i},\quad \mathbb{K}\in\{\mathbb{R},\mathbb{C}\}$
mapping states to $D_i$-dimensional hypervectors.  For each action $a$, it maintains a parameter hypervector $\mathbf w_{i,a}\in\mathbb{K}^{D_i}$. We use $\langle \cdot, \cdot \rangle$
to denote the standard Hermitian inner product on $\mathbb{K}^{D_i}$,
$\langle u, v \rangle = u^{\mathrm{H}} v = \sum_{k=1}^{D_i} \overline{u_k}\,v_k$,
and we write $\Re(z)$ for the real part of $z \in \mathbb{C}$, with the
convention $\Re(z) = z$ when $z \in \mathbb{R}$. The approximated action-value is 
\begin{equation}
Q_i(s,a)\;=\;\Re\!\bigl(\langle \Phi_i(s), \mathbf w_{i,a}\rangle\bigr)\,, 
\label{eq:qhd_q}
\end{equation}
so that stacking $W_i=[\mathbf w_{i,a}]_{a\in\mathcal{A}}\in\mathbb{K}^{D_i\times|\mathcal{A}|}$ yields 
$Q_i(s,\cdot)=\Re(\Phi_i(s)^{\mathrm H}W_i)$.  Although $\Phi_i$ may be nonlinear in $s$, the model \eqref{eq:qhd_q} is linear in parameters. QHD updates $W_i$ via standard off-policy TD with a delayed target network~\citep{ni2022qhd}; the explicit semi-gradient update rule is given in Appendix~\ref{app:qhd_update}. 


\paragraph{Federated value learning objective.}  We consider $N$ agents (clients) learning in parallel from private data. They do not share raw data; instead, after local training, each client $i$ has learned a Q-function $Q_i(s,a)=\Re(\Phi_i(s)^{\mathrm H}W_i)$.  The server defines a global \emph{function-space} objective by aggregating these models under an agreement distribution $\mu$ (implemented via a shared anchor set; see below).  

Concretely, given weights $\{\pi_i\}$ with $\sum_i\pi_i=1$, the server seeks the global value function $Q^{\mathrm{glob}}$ that minimizes 
\begin{equation}\label{eq:global}
\sum_{i=1}^N \pi_i\,\mathbb{E}_{(s,a)\sim\mu}\bigl[(Q_i(s,a)-Q(s,a))^2\bigr].
\end{equation}
By standard projection arguments, the minimizer is the weighted average in function space:
\begin{equation}
Q^{\text{glob}}(s,a)=\sum_i \pi_i Q_i(s,a), \quad (\mu\text{-a.e.}).
\label{eq:qglob_avg}
\end{equation}
Here $\mu$ denotes the agreement distribution over state-action pairs; in practice, we implement $\mu$ via a discrete anchor set.
The server thus obtains a global action-value model in closed form without accessing client trajectories.
The problem is to compile this function-space teacher into a parameter vector for each client’s representation.


Although Eq.~\eqref{eq:global} is defined in function space, FedAvg-style aggregation either requires iterative optimization (for nonlinear approximators) or becomes algebraically undefined (under heterogeneous encoders). FedQHD resolves both by exploiting the linearity in parameters of \eqref{eq:qhd_q}.

\section{FedQHD: Function-Space Federation with Closed-Form Compilation}
\label{sec:method}
We present FedQHD in two regimes. When clients share an encoder, the federation step is exact in closed form. When clients use heterogeneous encoders, the server aggregates client predictions on a shared anchor-state set, and each client compiles the resulting teacher into its local feature space via a closed-form ridge solve. 
%

\subsection{Homogeneous encoders}
\label{sec:homo_fedqhd}
If all clients share an encoder $\Phi$, substituting into \eqref{eq:qglob_avg} gives the global model $Q^{\mathrm{glob}}(s,a)=\Re\bigl(\Phi(s)^\mathrm{H}W^{\mathrm{glob}}\bigr)$, where $W^{\mathrm{glob}}\in\mathbb{K}^{D\times|\mathcal{A}|}$ is the minimizer of a quadratic objective $\sum_{i=1}^N \pi_i E_{(s,a)\sim\mu} \big[ \langle \Phi(s), (W_i -W^{\mathrm{glob}})_{:,a} \rangle^2\big]$. The closed-form solution is:
\[
W^{\mathrm{glob}} \;=\; \sum_{i=1}^N \pi_i W_i.
\]
That is, when $\Phi_i=\Phi$, function-space consensus coincides exactly with parameter averaging, which is independent of $\mu$.  The server simply computes the weighted average of the local weight matrices without requiring raw data and sends $W^{\mathrm{glob}}$ back to the clients as their updated parameters. The homogeneous FedQHD procedure is presented as pseudocode in Appendix~\ref{app:alg}.

\subsection{Heterogeneous encoders} 
\label{sec:hetero_fedqhd}

When clients use different encoders $\Phi_i$ with possibly different dimensions $D_i$, direct parameter averaging is not defined. Instead, we align client representations through a \emph{shared anchor set}. The server samples a set of reference states $\mathcal{S}_{\mathrm{ref}}=\{s_1,\dots,s_m\}$. The anchor set can be obtained from random rollouts, a shared unlabeled dataset, or states encountered during local training; in our experiments we use random rollouts.

Each client $i$ evaluates its Q-function on the anchors,
$Q^{\mathrm{ref}}_i(s_j,a)=Q_i(s_j,a)$, forming a matrix $Q^{\mathrm{ref}}_i\in\mathbb{R}^{m\times|\mathcal{A}|}$. The server aggregates these predictions to obtain the anchor teacher $Q^{\mathrm{glob}}_{\mathrm{ref}}
=\sum_{i=1}^N \pi_i Q^{\mathrm{ref}}_i.$ Client $i$ then fits parameters $W_i^{\mathrm{glob}}$ by solving $\min_{W\in\mathbb{K}^{D_i\times|\mathcal{A}|}}
\sum_{j=1}^m\sum_{a\in\mathcal{A}}
\Bigl(\Re(\Phi_i(s_j)^{H}w_a)-Q^{\mathrm{glob}}_{\mathrm{ref}}(s_j,a)\Bigr)^2
+\lambda\|W\|_F^2.$ Since the model is linear in $W$, the solution has the closed form
\[
W^{\mathrm{glob}}_i=
\left(X_i^{H}X_i+\lambda I_{D_i}\right)^{-1}X_i^{H}Q^{\mathrm{glob}}_{\mathrm{ref}},
\]
when $\lambda>0$ or $X_i^{H}X_i$ is full rank. 
This corresponds to ridge regression projecting the global teacher onto client $i$’s feature space. After updating to $W_i^{\mathrm{glob}}$, the client resumes local RL updates until the next federation round. Equivalently, using the Woodbury identity, $W_i^{\mathrm{glob}} = X_i^H(G_i + \lambda I_m)^{-1} Q_{\mathrm{ref}}^{\mathrm{glob}}$, which requires inverting an $m\times m$ matrix rather than $D_i\times D_i$. Algorithm~\ref{alg:fedqhd_heterogeneous} is provided in Appendix~\ref{app:alg}.



As analyzed in Sec.~\ref{sec:theory}, the resulting compilation error depends on representation mismatch, anchor conditioning, and regularization.

\section{Theoretical Analysis}
\label{sec:theory}
We analyze the static error induced when a federated teacher is compiled into a client-specific linear representation. The results characterize representation/compilation error conditional on fixed local predictors; they are not a convergence theorem for the full online FedQHD training dynamics.


\subsection{Federated Representation Mismatch}
\label{subsec:fed_gap}
Let $Q^*$ denote the true optimal action-value function. For each client $i$, define its \emph{oracle projection} $\hat Q_i$ as the best possible approximation to $Q^*$ within its function class:
\[
\hat Q_i(s,a) \;\defeq\; \arg\min_{Q\in\mathcal{F}_i} \mathbb{E}_{(s,a)\sim\mu} \bigl(Q(s,a) - Q^*(s,a)\bigr)^2,
\]
where $\mathcal{F}_i \defeq \{Q_i(s,a)=\Re(\Phi_i(s)^\mathrm{H} W_i) : W_i\in\mathbb{K}^{D_i\times|\mathcal{A}|}\}$ is the client’s QHD function class. In linear approximation theory, this is the orthogonal projection of $Q^*$ onto the span of $\Phi_i$, i.e., the best representation of the true value function that client $i$ can achieve given its encoder.

Equivalently, client $i$’s oracle parameters $\hat W_i$ minimize the mean-squared Bellman error for policy evaluation.  We then define the federation gap
\begin{equation}
\Delta_i(s,a) \;\defeq\; \hat Q_i(s,a) \;-\; Q_i\bigl(s,a;\,W_i^\mathrm{glob}\bigr),
\label{eq:gap_def}
\end{equation}
where $W_i^\mathrm{glob}$ is the global-compiled weight returned to the client $i$.  Thus $\Delta_i$ measures how far FedQHD’s aggregated model falls short of the best possible value in client $i$’s function class.


\subsection{Projection Residual Invisibility in Function Space}

Even if the global teacher $Q^{\mathrm{glob}}_{\mathrm{ref}}$ aggregates 
knowledge from all clients, only the portion representable within client 
$i$'s function class is absorbed during compilation. Formally, let 
$P_i$ denote the orthogonal projector onto the subspace spanned by 
client $i$'s anchor evaluations, and define the \emph{projection 
residual} $R_{i,0} = (I - P_i)Q^{\mathrm{glob}}_{\mathrm{ref}}$ as the 
component of the global teacher that lies outside this subspace.
The following result shows that the ridge solve automatically discards the unrepresentable portion of the teacher.

\begin{theorem}[Projection residual invisibility]
\label{thm:invisibility}
For any $\lambda \geq 0$, the compiled Q-function depends only on the 
in-subspace component of the global teacher:
\[
    Q_i\bigl(s,a;\,W_i^{\mathrm{glob}}(\lambda)\bigr)
    \;=\;
    \mathbf{k}_i(s)^{\top}
    (G_i + \lambda I)^{-1}
    P_i\,Q^{\mathrm{glob}}_{\mathrm{ref},a}
    \qquad \forall\, s \in \mathcal{S},\; a \in \mathcal{A},
\]
where $[\mathbf{k}_i(s)]_\ell = \langle \Phi_i(s_\ell), \Phi_i(s) \rangle$ 
for $\ell = 1, \ldots, m$, and $G_i \in \mathbb{K}^{m 
\times m}$ is the anchor Gram matrix with $[G_i]_{\ell j} = \langle 
\Phi_i(s_\ell), \Phi_i(s_j) \rangle$. 
Hence the projection residual is invisible: $\mathbf{k}_i(s)^{\top}(G_i + \lambda I)^{-1} R_{i,0,a} = 0$.

\end{theorem}

\begin{proof}
See Appendix~\ref{app:proof_invisibility}.
\end{proof}

This theorem shows that the compiled predictor depends only on the component of the teacher that lies in client $i$'s anchor-feature subspace. Representation mismatch, therefore, appears as unrepresented residual information rather than direct contamination of the fitted coefficients. The resulting loss is quantified in Theorem~\ref{thm:fed_gap_main}.

\subsection{Non-Asymptotic Geometric Decomposition of the Federation Gap}
\label{subsec:main_theorem}

We give a non-asymptotic, pointwise bound on the federation gap defined in Sec.~\ref{subsec:fed_gap} under encoder heterogeneity. Here we introduce standard regularity assumptions ensuring that
(i) all predictors lie in a common function space, (ii) client features are non-degenerate,
and (iii) the anchor-based ridge regression is well posed.


\begin{assumption}[Common RKHS]
\label{assum:common_rkhs}
There exists a continuous, symmetric, positive-definite kernel
$\kappa:\mathcal{S}\times\mathcal{S}\to\mathbb{R}$ with
$\kappa(s,s)\leq 1$ for all $s\in\mathcal{S}$, inducing a separable
RKHS $\mathcal{H}$ such that:
\begin{enumerate}
    \item \textup{(RKHS norm bound)} $\|Q^*(\cdot,a)\|_{\mathcal{H}}
    \leq B < \infty$ for all $a\in\mathcal{A}$. 
    This is a standard regularity condition in kernel-based RL~\citep{jin2020provably}; it is implied when $Q^\star$ lies in the RKHS of $\kappa$, which holds under standard smoothness conditions on the MDP dynamics and reward;
    \item \textup{(Common embedding)} all client function classes embed
    as $\mathcal{F}_i \hookrightarrow \mathcal{H}$, with the reproducing
    property $f(s) = \langle f, \kappa(\cdot,s)\rangle_{\mathcal{H}}$
    holding for all $f\in\mathcal{H}$, $s\in\mathcal{S}$;
    \item \textup{(Feature boundedness)} $\|\Phi_i(s)\|_2 \leq 1$
    for all $s\in\mathcal{S}$, $i=1,\ldots,N$.
\end{enumerate}
\end{assumption}

\begin{assumption}[Feature covariance]
\label{assum:feature_cov}
For each client $i$, let $\nu_i$ denote the marginal state distribution of 
client $i$'s local experience. The feature covariance
$\Sigma_i \defeq \mathbb{E}_{s\sim\nu_i}[\Phi_i(s)\Phi_i(s)^{\mathrm{H}}]
\succ 0$. 

This holds generically for random Fourier feature (RFF) encoders with
continuous state distributions, since the induced features span a
$D_i$-dimensional subspace with probability one when
$|\mathrm{supp}(\nu_i)| \ge D_i$.
\end{assumption}

\begin{assumption}[Anchor set]
\label{assum:anchor}
The anchor set $\mathcal{S}_{\mathrm{ref}}=\{s_1,\ldots,s_m\}$ is
drawn i.i.d.\ from a distribution $\nu_{\mathrm{ref}}$ with
$\nu_{\mathrm{ref}}\ll\nu_i$ for all $i$, and $\lambda>0$.
\end{assumption}

\begin{remark}
Assumption~\ref{assum:anchor} requires $\nu_{\mathrm{ref}} \ll \nu_i$ for all clients,
ensuring the anchor set covers each client's state support. In practice, this is
approximated by drawing anchors from random rollouts in the shared environment. When clients operate in distinct regions of $\mathcal{S}$, the anchor set should be enlarged to cover all relevant regions.
\end{remark}

Under these assumptions, we first isolate the component of error that arises solely from representation heterogeneity, independent of the anchor-based compilation step. Lemma~\ref{lem:term_A} quantifies how geometric subspace misalignment translates into disagreement between client oracle predictors.

\begin{lemma}[Representation bias]
\label{lem:term_A}
Under Assumption~\ref{assum:common_rkhs}, for any client $i$ and action $a$,
\[
|\hat Q_i(s,a)-\bar Q(s,a)| \;\le\; 2B\sum_{j\neq i}\pi_j\sin(\theta_{ij})
\;+\;\sum_{j\neq i}\pi_j(\varepsilon_i^{\mathrm{rep}}+\varepsilon_j^{\mathrm{rep}}),
\]
where $\bar Q=\sum_{j}\pi_j\hat Q_j$ is the weighted oracle mean and $\theta_{ij}$ is the largest principal angle between the subspaces of clients $i$ and $j$. This term corresponds to component~(I) in the
federation-gap decomposition.
\end{lemma}


Lemma~\ref{lem:term_A} characterizes the irreducible heterogeneity floor:
even with perfect anchor conditioning and $\lambda=0$, federation cannot
eliminate oracle disagreement induced by subspace misalignment and
representation error. This geometric term forms the baseline component of the final federation-gap bound in Theorem~\ref{thm:fed_gap_main}.

\begin{theorem}[Federation gap bound]
\label{thm:fed_gap_main}
Under Assumptions~\ref{assum:common_rkhs}--\ref{assum:anchor}, for each client $i$, any $(s,a)$, and regularization $\lambda\ge0$,  define the anchor feature matrix
$X_i \in \mathbb{K}^{m \times D_i}$ by stacking the encoded anchor 
states row-wise, $[X_i]_{\ell,:} = \Phi_i(s_\ell)^\top$ for 
$\ell = 1,\ldots,m$, and the anchor Gram matrix
$G_i = X_i X_i^{\mathrm{H}} \in \mathbb{K}^{m \times m}$ with entries
$[G_i]_{\ell j} = \langle \Phi_i(s_\ell), \Phi_i(s_j)\rangle$. The aggregation error defined in \eqref{eq:gap_def} satisfies
\begin{equation}
    |\Delta_i(s,a)|
\;\le\;
\underbrace{\bar h_i}_{\textbf{(I) encoder heterogeneity}}
\;+\;\underbrace{\frac{\sqrt{m}}{\sqrt{\gamma_i+\lambda}}\;\bar h_i}_{\textbf{(II) anchor amplification}}
\;+\;\underbrace{\frac{\lambda}{\gamma_i+\lambda}\,\|\hat W_i\|_F}_{\textbf{(III) ridge shrinkage}}
\end{equation}

where $\gamma_i \;\defeq\; \lambda_{\min}^{+}(G_i)$ is the smallest positive eigenvalue of client $i$’s anchor feature Gram matrix, $\bar h_i = \sum_{j\neq i}\pi_j\,(2B\sin(\theta_{ij}) + \varepsilon_i^{\mathrm{rep}}+\varepsilon_j^{\mathrm{rep}})$, and $\|\hat W_i\|_F$ is client $i$'s oracle parameters.
\end{theorem}

\begin{proof}
    See Appendix~\ref{app:proof_termB}.
\end{proof}

\begin{remark}
\label{rem:m_scaling}
 Although Term~(II) appears to grow with $\sqrt{m}$, this is an artifact of converting RKHS norms to finite-dimensional $\ell_2$ norms over anchor evaluations. When $m \geq D_i$, $\gamma_i$ scales linearly
with $m$ under Assumption~\ref{assum:anchor}, so $\sqrt{m}/\sqrt{\gamma_i+\lambda}$
is asymptotically a constant independent of $m$, and Term~(II) reduces
to a multiple of $\bar h_i$.
\end{remark}

From Theorem~\ref{thm:fed_gap_main} we obtain several corollaries:
\begin{corollary}[Zero-gap conditions]
$\Delta_i = 0$ iff all clients' subspaces align ($\theta_{ij}=0$) and
have no representation error, $\lambda \to 0$, and the anchor feature matrix
    has full column rank ($\mathrm{rank}(X_i) = D_i$).
\end{corollary}
 


\begin{corollary}[Heterogeneity-dominated regime]
\label{col:phase_transition}
Suppose $m \geq D_i$, so that $G_i$ has full rank $D_i$ on
$\mathrm{col}(X_i)$. By Remark~\ref{rem:m_scaling}, $\sqrt{m}/\sqrt{\gamma_i+\lambda}$
is asymptotically $m$-independent. Sending $\lambda \to 0$
in Theorem~\ref{thm:fed_gap_main} collapses the ridge shrinkage
Term~(III), leaving
$    |\Delta_i(s,a)|
    \;\le\;
    \mathcal{O}(\bar h_i).$
In this regime the federation gap is governed entirely by the encoder
heterogeneity $\bar h_i$, and is reduced primarily by smaller $\sin\theta_{ij}$ or
$\varepsilon_i^{\mathrm{rep}}$.
\end{corollary}


Theorem~\ref{thm:fed_gap_main} decomposes the federation gap into three
interpretable terms: (I) an irreducible encoder heterogeneity $\bar h_i$
driven by principal-angle misalignment and per-client representation
error, (II) an anchor-amplified term whose apparent $\sqrt{m}$ growth
cancels asymptotically when $m \geq D_i$, and (III) a ridge
shrinkage term controlled by $\lambda$. Consistent with the theory, experiments (Sec.~\ref{sec:experiments}) show that under RFF encoder heterogeneity with $m \geq D_i$, the gap
collapses to $\mathcal{O}(\bar h_i)$.

%% file: sections/experiments.tex
\section{Experiments}
\label{sec:experiments}

We evaluate FedQHD on four continuous-control benchmarks under both
homogeneous and heterogeneous encoder settings. Our experiments address:
\textbf{(Q1)} Does FedQHD improve policy learning over independent
training and existing federated RL baselines?
\textbf{(Q2)} Under heterogeneous encoders, can anchor-based aggregation
recover a meaningful shared value function?
\textbf{(Q3)} Does FedQHD reduce computation compared with
backpropagation-based and distillation-based federation?
\textbf{(Q4)} How does the encoder dimension D affect FedQHD quality?

Learning curves on CartPole and LunarLander, the ablation on anchor set size $m$, the scalability study
with respect to $N$, and the full experimental setup are provided in Appendix~\ref{app:exp}--\ref{subsec:scalability}.



\subsection{Experimental Setup}
\label{subsec:exp_setup}

We conduct FedQHD on four continuous-state and discrete-action control benchmarks from OpenAI Gym~\citep{brockman2016openai}: \texttt{CartPole-v1}, \texttt{Acrobot-v1}, \texttt{LunarLander-v3} and \texttt{MountainCar-v0}.

In homogeneous experiments, all clients share a random Fourier feature (RFF) encoder of dimension $D=10{,}000$ and fixed bandwidth $\sigma$. 
In heterogeneous experiments, each client independently samples an RFF encoder with bandwidth $\sigma_i \sim \mathrm{Unif}[0.5\sigma_0,1.5\sigma_0]$ and dimension $D_i$ ranging from $5\times10^2$ to $10^4$. Anchor-based aggregation uses a server-constructed reference set of size $m=200$. Unless otherwise stated, we use \(m=200\) anchors in the main experiments; Appendix~\ref{app:ablation_m} studies the effect of varying \(m\).

We compare FedQHD against the following baselines:
{(i) Independent QHD: The performance of a randomly selected local model in each of the $N$ involved environments.} {(ii) Oracle QHD$^\dagger$: A single QHD trained on data pooled from all clients.} { (iii) Oracle DQN$^\dagger$: A single DQN trained on pooled data from all clients.} {(iv) FedAvg-DQN~\citep{jin2022federated}: Federated deep Q-learning with parameter averaging.} { (v) Truncate FedAvg-QHD: Client weight matrices $W_i \in \mathbb{R}^{D_i \times |A|}$ are truncated to the minimum dimension $D_{\min}$ before averaging, then zero-padded back to each client's original dimension.}
{(vi) Distillation FedDQN \citep{jiang2025fedhpd}: Heterogeneous aggregation by distilling a client DQN toward an anchor-set teacher updated via averaged soft predictions.}


We report (i) average reward: mean episodic return across
clients; (ii) compilation error: maximum Q-value deviation
between FedQHD and the client oracle on held-out anchors; and
(iii) policy value gap: the return difference between greedy
policies induced by Oracle QHD and FedQHD.

%

\subsection{Results}
\label{sec:results}

\paragraph{Performance Comparison (Q1 and Q2)}
\label{subsec:performance_homo}

Table~\ref{tab:main_results} reports final average reward after training
$N=5$ clients for 600 episodes. Full results with standard deviations across 3 seeds are reported in Appendix~\ref{app:performance}. FedQHD achieves the best non-oracle
performance in 5 of 8 tasks and even surpasses Oracle QHD on
LunarLander. 


Under heterogeneous encoders, aggregation strategies diverge sharply.
Truncate FedAvg-QHD collapses when moving from Q1 to Q2 (e.g.,
CartPole drops 73\%), confirming that naive dimension matching
destroys the RFF feature structure. In contrast, FedQHD remains
competitive because anchor-based projection preserves the geometry
of client representations. Distillation FedDQN also transfers across
heterogeneous architectures but requires iterative optimization and
is substantially slower (Table~\ref{tab:computation_cost}), whereas
FedQHD performs a single closed-form compilation step.

\begin{table}[ht]
\centering
\caption{Final average reward (last 100 episodes), mean over $N{=}5$ clients and 3 seeds. \textbf{Bold}: best non-oracle. $\dagger$: oracle (pooled data). {--}: not applicable.}
\label{tab:main_results}
\setlength{\tabcolsep}{4pt}
\small
\begin{tabular}{l rr rr rr rr}
\toprule
& \multicolumn{2}{c}{CartPole} & \multicolumn{2}{c}{Acrobot} & \multicolumn{2}{c}{LunarLander} & \multicolumn{2}{c}{MountainCar} \\
\cmidrule(lr){2-3}\cmidrule(lr){4-5}\cmidrule(lr){6-7}\cmidrule(lr){8-9}
Method & Homo & Hetero & Homo & Hetero & Homo & Hetero & Homo & Hetero \\
\midrule
Indep. QHD          & 26.1            & 24.0            & $-306.1$        & $-495.4$        & $-187.3$        & $-174.6$        & $-200.0$        & $-200.0$ \\
Trunc. FedAvg-QHD   & 230.8           & 112.6           & $-106.5$        & $-280.8$        & 65.0            & $-24.8$         & $-148.3$        & $-172.9$ \\
FedAvg-DQN          & 120.8           & {--}            & $\mathbf{-85.3}$ & {--}           & 200.9           & {--}            & $-167.8$        & {--} \\
Distill. FedDQN     & 169.9           & 149.4           & $-87.5$         & $\mathbf{-89.2}$ & 122.5         & $\mathbf{131.4}$ & $-156.0$       & $-174.3$ \\
FedQHD              & $\mathbf{466.3}$ & $\mathbf{351.1}$ & $-105.0$       & $-102.5$        & $\mathbf{224.1}$ & 101.5          & $\mathbf{-143.7}$ & $\mathbf{-162.3}$ \\
\midrule
Oracle QHD$^\dagger$ & 373.2          & 458.5           & $-93.6$         & $-101.2$        & 87.6            & 232.9           & $-141.3$        & $-140.0$ \\
Oracle DQN$^\dagger$ & 139.1          & 204.4           & $-75.3$         & $-79.8$         & 183.5           & 202.6           & $-125.7$        & $-122.8$ \\
\bottomrule
\end{tabular}
\end{table}


\paragraph{Computation Cost.} Table~\ref{tab:computation_cost} reports the wall-clock training time of FedQHD with 5 clients after 600 episodes, complementing the performance results in Table~\ref{tab:main_results}.  FedQHD is consistently faster than all DQN-based methods, as it entirely replaces backpropagation with closed-form TD(0) updates and, unlike Distillation FedDQN, requires no additional per-round gradient loops through a teacher network. 
Under homogeneous encoders with $D=10{,}000$, FedQHD takes 9.7--20.3 min across environments, compared to 25.9--85.6 min for FedAvg-DQN and 14.3--69.2 min for Distillation FedDQN. A more revealing finding emerges under Q2: despite requiring an extra anchor-based ridge-regression solve at each aggregation round, FedQHD's Q2 wall-clock time (1.4--5.9 min) is $3$--$12\times$ lower than its own Q1 cost on the same environment. This is because Q2 clients operate with mixed encoder dimensions ($D_i$ ranging from $5\times10^2$ to $10^4$, assigned cyclically) versus the fixed $D=10{,}000$ used uniformly in Q1,  demonstrating that in FedQHD representation dimensionality is the primary driver of computation cost.



  \begin{table}[t]
\centering
\caption{%
  Wall-clock training time (minutes per run, 600 episodes,  $N{=}5$ clients) reported as mean $\pm$ std across 3 seeds.
  CP = CartPole, Acr = Acrobot, LL = LunarLander, MC = MountainCar.
}
\label{tab:computation_cost}
\setlength{\tabcolsep}{3pt}
\small

\begin{tabular}{l rrrr rrrr}
\toprule
 & \multicolumn{4}{c}{Homogeneous}
 & \multicolumn{4}{c}{Heterogeneous} \\
\cmidrule(lr){2-5}\cmidrule(lr){6-9}
Method & CP & Acr & LL & MC & CP & Acr & LL & MC \\
\midrule
Independent QHD
  & $1.7{\scriptstyle\pm0.1}$  & $9.6{\scriptstyle\pm0.1}$  & $2.8{\scriptstyle\pm0.0}$  & $12.7{\scriptstyle\pm0.1}$
  & $0.1{\scriptstyle\pm0.0}$  & $5.0{\scriptstyle\pm0.0}$  & $0.8{\scriptstyle\pm0.0}$  & $1.4{\scriptstyle\pm0.0}$ \\
Truncate FedAvg-QHD
  & $3.5{\scriptstyle\pm0.2}$  & $5.6{\scriptstyle\pm0.0}$  & $15.5{\scriptstyle\pm0.2}$ & $7.6{\scriptstyle\pm0.2}$
  & $0.3{\scriptstyle\pm0.0}$  & $2.1{\scriptstyle\pm0.3}$  & $2.9{\scriptstyle\pm0.1}$  & $1.7{\scriptstyle\pm0.1}$ \\
FedAvg-DQN
  & $25.9{\scriptstyle\pm1.4}$ & $22.8{\scriptstyle\pm0.7}$ & $36.8{\scriptstyle\pm0.6}$ & $85.6{\scriptstyle\pm4.8}$
  & {--} & {--} & {--} & {--} \\
Distillation FedDQN
  & $12.4{\scriptstyle\pm0.6}$ & $21.6{\scriptstyle\pm0.1}$ & $31.3{\scriptstyle\pm5.5}$ & $28.1{\scriptstyle\pm2.4}$
  & $10.3{\scriptstyle\pm1.4}$ & $19.7{\scriptstyle\pm2.3}$ & $35.2{\scriptstyle\pm3.7}$ & $52.9{\scriptstyle\pm3.6}$ \\
FedQHD
  & $\phantom{0}7.8{\scriptstyle\pm0.4}$ & $\phantom{0}5.6{\scriptstyle\pm0.0}$ & $15.5{\scriptstyle\pm0.2}$ & $17.2{\scriptstyle\pm0.3}$
  & $1.6{\scriptstyle\pm0.1}$  & $1.9{\scriptstyle\pm0.0}$  & $5.9{\scriptstyle\pm0.1}$            & $1.4{\scriptstyle\pm0.0}$ \\
\midrule
Oracle QHD$^\dagger$
  & $\phantom{0}6.1{\scriptstyle\pm0.7}$ & $\phantom{0}4.9{\scriptstyle\pm0.0}$ & $15.3{\scriptstyle\pm0.3}$ & $\phantom{0}7.7{\scriptstyle\pm0.3}$
  & $6.1{\scriptstyle\pm0.7}$  & $6.3{\scriptstyle\pm0.1}$  & $14.8{\scriptstyle\pm0.4}$ & $8.5{\scriptstyle\pm0.4}$ \\
Oracle DQN$^\dagger$
  & $\phantom{0}8.9{\scriptstyle\pm0.3}$ & $15.2{\scriptstyle\pm1.9}$ & $24.9{\scriptstyle\pm1.2}$ & $67.8{\scriptstyle\pm3.3}$
  & $64.1{\scriptstyle\pm3.7}$ & $77.3{\scriptstyle\pm7.8}$ & $130.4{\scriptstyle\pm7.9}$ & $101.5{\scriptstyle\pm6.9}$ \\
\bottomrule
\end{tabular}
\end{table}

\paragraph{Effect of Encoder Dimension $D$.} We study CartPole and vary $D$ from 16 to 2048 while fixing $m=4D$, keeping all runs in the over-determined regime ($m>D$) so that the RFF feature map and the encoder heterogeneity dominates compiled error following the geometry floor $|\Delta_i| = \mathcal{O}(D^{-1/2})$~\citep{rahimi2007random} .

Figure~\ref{fig:ablation1} confirms this prediction in Corollary~\ref{col:phase_transition}. The center panel shows the compiled error with a -0.525 log--log slope, validating that FedQHD's anchor--based projection inherits the standard RFF approximation guarantee. The left panel shows that lower $D$ leads to both slower convergence and lower final reward, while $D \geq 128$ consistently reaches the CartPole ceiling ($V \approx 500$) within 600 episodes. The right panel directly links the policy value gap to the compiled error curve, confirming that the federation gap is a reliable proxy for practical policy degradation.

\begin{figure}[t]
  \centering
  \includegraphics[width=\linewidth]{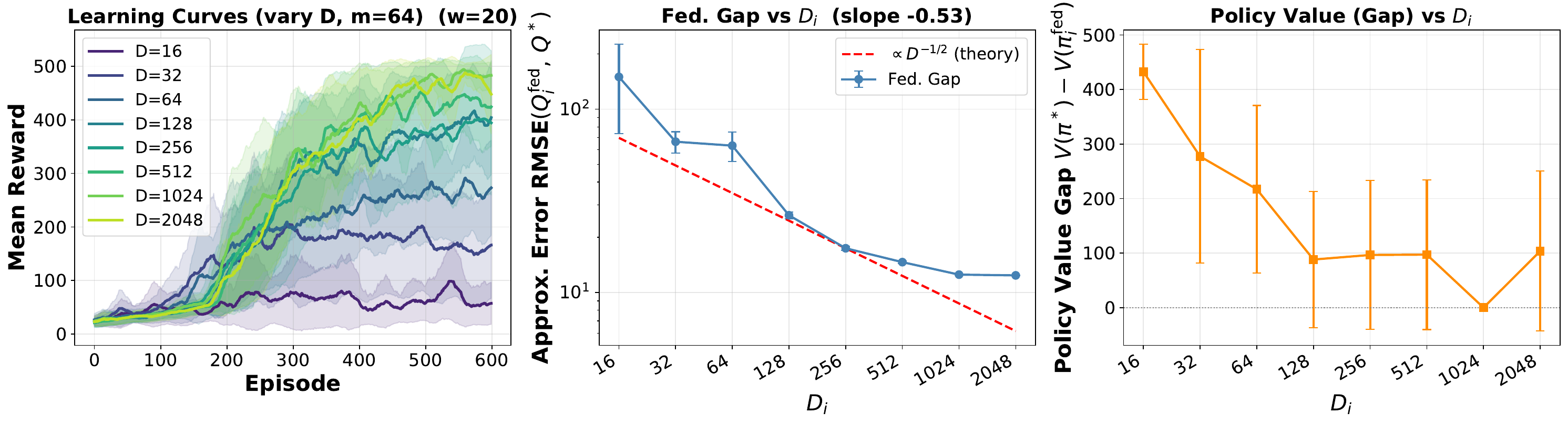}
  \caption{Ablation 1 (vary $D$, set $m=4D$): learning curves (left), Q-error vs.\ $D$ (middle), and final policy value (right).}
  \label{fig:ablation1}
\end{figure}

%% file: sections/conclusion.tex
\section{Conclusion}
\label{sec:conclusion}

We presented FedQHD, a federated $Q$-learning framework that replaces
parameter-space synchronization with closed-form function-space
aggregation for linear-in-parameter hyperdimensional (random-feature)
value representations. With a shared encoder, FedQHD reduces exactly to
weighted averaging of local readout matrices, recovering FedAvg in
closed form. With heterogeneous encoders, the server aggregates client
$Q$-values on a shared anchor-state interface, and each client compiles
the resulting teacher into its local representation via a one-shot
ridge projection---avoiding per-round iterative distillation and
gradient-based teacher--student optimization. We further derived a
pointwise bound on the federation gap, decomposing it into encoder
heterogeneity, anchor-set conditioning, and ridge shrinkage, and
identified $m \geq D_i$ as the well-conditioned regime in which the gap
reduces to a multiple of the heterogeneity floor. On four
continuous-state, discrete-action benchmarks, FedQHD matches or exceeds
federated DQN baselines while requiring substantially less computation,
with the empirical dependence of the federation gap on encoder
dimension matching our theoretical analysis.
Extending closed-form heterogeneous federation to learned encoders,
richer observation spaces, and actor--critic settings remains an
important direction for future work.

%% file: appendix/exp.tex
\section{QHD Semi-Gradient Update}
\label{app:qhd_update}
 
For completeness, we record the semi-gradient TD update rule used by
QHD~\citep{ni2022qhd}. Given a transition $(s, a, r, s')$, a delayed
target weight $W_i^-$, and learning rate $\eta$, define the bootstrapped
target
\[
    y \;=\; r + \gamma\,\Re\!\bigl(\Phi_i(s')^{\mathrm H}\,\mathbf w^{-}_{i,a^\star}\bigr),
    \qquad
    a^\star = \arg\max_{a'\in\mathcal{A}} Q_i(s', a').
\]
The semi-gradient update on the readout vector for action $a$ is
\[
    \mathbf w_{i,a} \;\leftarrow\; \mathbf w_{i,a}
    + \eta\,\bigl(y - Q_i(s,a)\bigr)\,\Phi_i(s).
\]
The target weights $W_i^-$ are periodically synchronized with $W_i$, in
the standard fashion of off-policy $Q$-learning with target networks.
 
\section{Algorithm}
\label{app:alg}
The pseudocode of FedQHD procedure is presented here.
\begin{figure*}[h]
\begin{minipage}[t]{0.48\textwidth}
\begin{algorithm}[H]
\caption{FedQHD: Homogeneous Protocol}
\label{alg:fedqhd_shared}
\begin{algorithmic}[1]
\STATE \textbf{Init:} shared encoder
       $\Phi:\mathcal{S}\to\mathbb{K}^D$;\;
       $W^{\mathrm{glob}}_0 = \mathbf{0}$
\FOR{round $k = 0, \ldots, T\!-\!1$}
    \STATE Broadcast $W^{\mathrm{glob}}_k$ to all clients
    \FOR{client $i = 1,\ldots,N$ \textbf{in parallel}}
        \STATE $W_i, W_i^- \leftarrow W^{\mathrm{glob}}_k$
        \FOR{$e = 1,\ldots,K$}
            \STATE $\delta \leftarrow \text{TD}(s,a,r,s'; W_i, W_i^-)$
            \STATE $W_i[:,a] \leftarrow 
           W_i[:,a] + \eta\, \delta\, \Phi(s)$
            \STATE Sync $W_i^-$ every $\tau$ eps; decay $\epsilon$
        \ENDFOR
        \STATE Upload $W_i$ to server
    \ENDFOR
    \STATE $W^{\mathrm{glob}}_{k+1} \leftarrow
           \sum_{i} \pi_i W_i$
\ENDFOR
\end{algorithmic}
\end{algorithm}
\end{minipage}
\hfill
\begin{minipage}[t]{0.48\textwidth}
\begin{algorithm}[H]
\caption{FedQHD: Heterogeneous Protocol}
\label{alg:fedqhd_heterogeneous}
\begin{algorithmic}[1]
\STATE \textbf{Init:} anchors
       $\mathcal{S}_{\mathrm{ref}}\!=\!\{s_1,\ldots,s_m\}$;
       each client $i$: encoder $\Phi_i$, $W_i\!=\!\mathbf{0}$
\FOR{round $k = 0, \ldots, T\!-\!1$}
    \FOR{client $i = 1,\ldots,N$ \textbf{in parallel}}
        \STATE $W_i^- \leftarrow W_i$
        \FOR{$e = 1,\ldots,K$}
            \STATE $\delta \leftarrow \text{TD}(s,a,r,s'; W_i, W_i^-)$
            \STATE $W_i[:,a] \leftarrow 
           W_i[:,a] + \eta\, \delta\, \Phi_i(s)$
            \STATE Sync $W_i^-$ every $\tau$ eps; decay $\epsilon$
        \ENDFOR
        \STATE Upload $Q^{\mathrm{ref}}_i\!=\!\Re(X_iW_i)$
    \ENDFOR
    \STATE \textbf{Server:}
           $Q^{\mathrm{glob}}_{\mathrm{ref}} \leftarrow
           \sum_{i} \pi_i Q^{\mathrm{ref}}_i$;
    \FOR{client $i$ \textbf{in parallel}}
        \STATE $W_i^{\mathrm{glob}} \leftarrow X_i^{\mathrm{H}}
               (G_i\!+\!\lambda I)^{-1}
               Q^{\mathrm{glob}}_{\mathrm{ref}}$ 
    \ENDFOR
\ENDFOR
\end{algorithmic}
\end{algorithm}
\end{minipage}
\end{figure*}

\section{Experimental setup.} 
\label{app:exp}
For homogeneous experiments, all clients share a random Fourier feature (RFF) encoder of dimension $D=10{,}000$ with fixed bandwidth $\sigma$:$
\Phi(s)=\tfrac{1}{\sqrt{D}}\big[\cos(\omega_1^\top s+b_1),\dots,\cos(\omega_D^\top s+b_D)\big]^\top$. For heterogeneous experiments, each client independently samples $\Phi_i$ with bandwidth $\sigma_i\!\sim\!\mathrm{Unif}[0.5\sigma_0,1.5\sigma_0]$, dimension $D_i\!\in\!\{500,1000,2000,5000,10000\}$ and anchor set $\mathcal{S}_{\text{ref}}$ with $m=200$ states collected from random rollouts.

All methods use $\epsilon$-greedy exploration with $\epsilon$ annealed from 1.0 to 0.001. We use learning rate $\eta=0.01$, discount $\gamma=0.99$, replay buffer size 10,000 per client, and uniform federation weights $\pi_i=1/N$. Federated aggregation occurs every $K=50$ local episodes with $N=5$ clients unless stated otherwise. We report mean over 3 random seeds. All DQN-based baseline use a 2-layer MLP with 128 hidden units per layer.

\section{Performance Comparison}
\label{app:performance}

\begin{table}[t]
\centering
\caption{Final average reward (last 100 episodes), mean $\pm$ std over $N{=}3$ seeds. \textbf{Bold}: best non-oracle. $\dagger$: oracle (pooled data). {--}: N/A.}
\label{tab:main_results_variance}
\resizebox{\textwidth}{!}{%
\begin{tabular}{l rr rr rr rr}
\toprule
& \multicolumn{2}{c}{CartPole} & \multicolumn{2}{c}{Acrobot} & \multicolumn{2}{c}{LunarLander} & \multicolumn{2}{c}{MountainCar} \\
\cmidrule(lr){2-3}\cmidrule(lr){4-5}\cmidrule(lr){6-7}\cmidrule(lr){8-9}
Method & Homo & Hetero & Homo & Hetero & Homo & Hetero & Homo & Hetero \\
\midrule
Indep. QHD
  & $26.1{\scriptstyle\pm1.1}$    & $24.0{\scriptstyle\pm1.1}$
  & $-306.1{\scriptstyle\pm17.1}$ & $-495.4{\scriptstyle\pm0.4}$
  & $-187.3{\scriptstyle\pm6.4}$  & $-174.6{\scriptstyle\pm5.6}$
  & $-200.0{\scriptstyle\pm0.0}$  & $-200.0{\scriptstyle\pm0.0}$ \\
Trunc. FedAvg-QHD
  & $230.8{\scriptstyle\pm85.7}$  & $112.6{\scriptstyle\pm26.1}$
  & $-106.5{\scriptstyle\pm1.8}$  & $-280.8{\scriptstyle\pm52.8}$
  & $65.0{\scriptstyle\pm5.1}$    & $-24.8{\scriptstyle\pm9.7}$
  & $-148.3{\scriptstyle\pm12.5}$ & $-172.9{\scriptstyle\pm18.6}$ \\
FedAvg-DQN
  & $120.8{\scriptstyle\pm9.0}$   & {--}
  & $\mathbf{-85.3{\scriptstyle\pm2.5}}$ & {--}
  & $200.9{\scriptstyle\pm14.0}$  & {--}
  & $-167.8{\scriptstyle\pm9.3}$  & {--} \\
Distill. FedDQN
  & $169.9{\scriptstyle\pm8.8}$   & $149.4{\scriptstyle\pm12.8}$
  & $-87.5{\scriptstyle\pm0.6}$   & $\mathbf{-89.2{\scriptstyle\pm2.6}}$
  & $122.5{\scriptstyle\pm53.2}$  & $\mathbf{131.4{\scriptstyle\pm18.4}}$
  & $-156.0{\scriptstyle\pm7.8}$  & $-174.3{\scriptstyle\pm14.6}$ \\
FedQHD
  & $\mathbf{466.3{\scriptstyle\pm22.3}}$ & $\mathbf{351.1{\scriptstyle\pm61.1}}$
  & $-105.0{\scriptstyle\pm3.3}$  & $-102.5{\scriptstyle\pm0.7}$
  & $\mathbf{224.1{\scriptstyle\pm11.8}}$ & $101.5{\scriptstyle\pm24.3}$
  & $\mathbf{-143.7{\scriptstyle\pm5.7}}$ & $\mathbf{-162.3{\scriptstyle\pm8.9}}$ \\
\midrule
Oracle QHD$^\dagger$
  & $373.2{\scriptstyle\pm80.9}$  & $458.5{\scriptstyle\pm52.1}$
  & $-93.6{\scriptstyle\pm0.6}$   & $-101.2{\scriptstyle\pm1.5}$
  & $87.6{\scriptstyle\pm6.2}$    & $232.9{\scriptstyle\pm14.3}$
  & $-141.3{\scriptstyle\pm4.1}$  & $-140.0{\scriptstyle\pm3.9}$ \\
Oracle DQN$^\dagger$
  & $139.1{\scriptstyle\pm10.5}$  & $204.4{\scriptstyle\pm16.9}$
  & $-75.3{\scriptstyle\pm0.4}$   & $-79.8{\scriptstyle\pm2.0}$
  & $183.5{\scriptstyle\pm12.8}$  & $202.6{\scriptstyle\pm15.7}$
  & $-125.7{\scriptstyle\pm6.3}$  & $-122.8{\scriptstyle\pm5.5}$ \\
\bottomrule
\end{tabular}
}
\end{table}
Figure~\ref{fig:learning_curve} shows that FedQHD
converges within 600 episodes on CartPole and LunarLander, whereas
DQN-based methods plateau earlier. This suggests that function-space
aggregation transfers more useful information per communication round
than parameter averaging in deep networks.

\begin{figure}[h!]
    \centering
    \includegraphics[width=\linewidth]{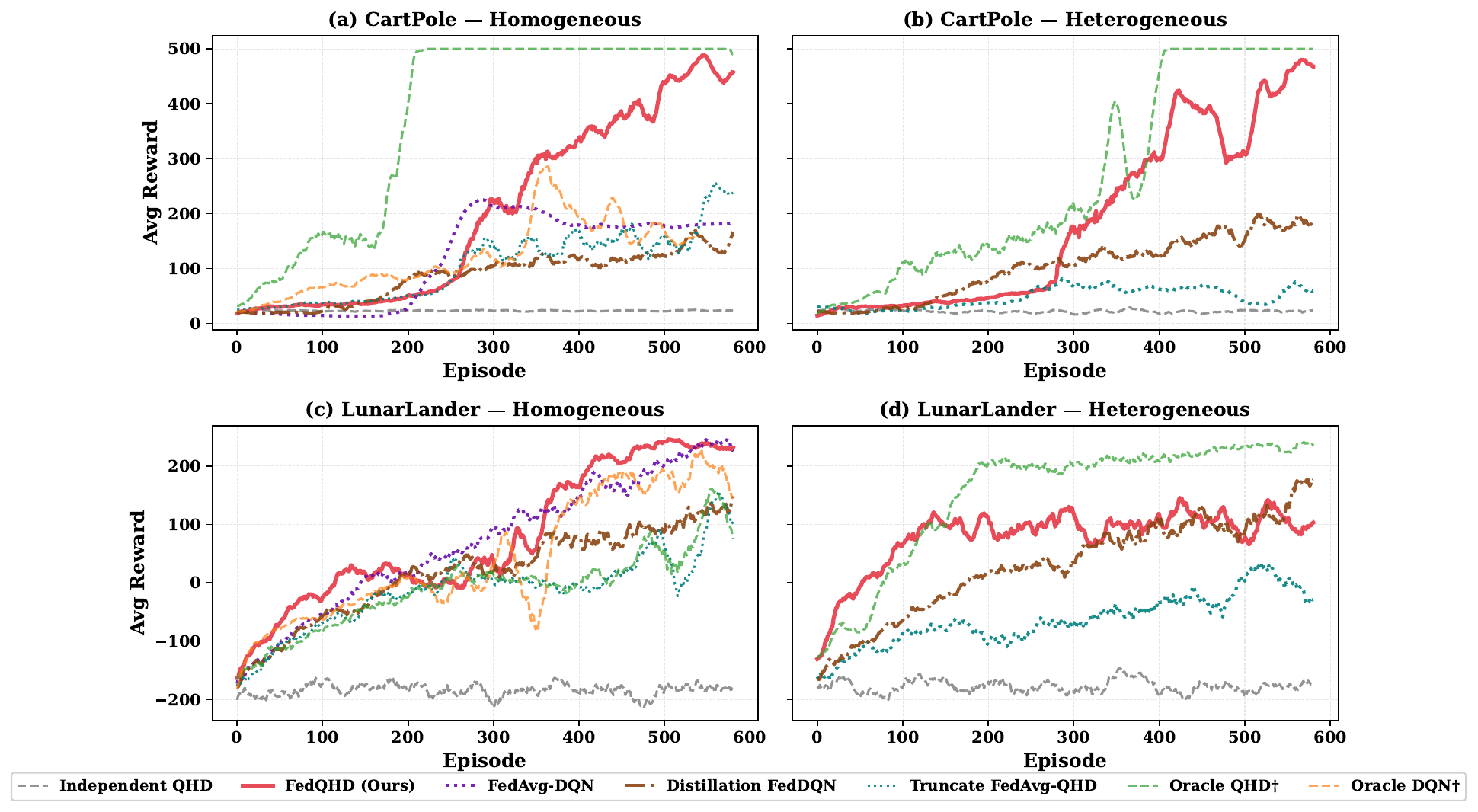}
    \caption{
    Learning curves for CartPole (top) and LunarLander (bottom) under homogeneous (left) and heterogeneous (right) encoders.}
    \label{fig:learning_curve}
\end{figure}

\section{Effect of Anchor Set Size $m$.}
\label{app:ablation_m}
We fix $D=512$ and vary the anchor size $m$ from $51$ to $2048$ on CartPole. Theorem~\ref{thm:fed_gap_main} predicts a regime change around $m=D$ when $m < D$, the anchor feature matrix is rank-deficient, Term~(II) is unbounded, and compiled weights generalize poorly. Once $m \geq D$, the residual reduces to the heterogeneity floor $\mathcal{O}(D^{-1/2})$.

Figure~\ref{fig:ablation2} shows the phase transition: the compiled error in the middle panel drops steeply within the under-determined regime; once $m\ge D$, the error stabilizes near the geometry floor, and FedQHD performance becomes reliable. The learning curve and policy value are unstable when $m < D$, but converge reliably to $V\approx500$ once $m \geq D$, reflecting the boundary sensitivity of the ridge solve.

Together with the ablation on dimension, this confirms the design rule $m \geq D$: the encoder dimension sets the geometry floor ($D^{-1/2}$ scaling), and the anchor count determines the projection regime, both predicted by Theorem~\ref{thm:fed_gap_main}.

\begin{figure}[ht]
  \centering
  \includegraphics[width=\linewidth]{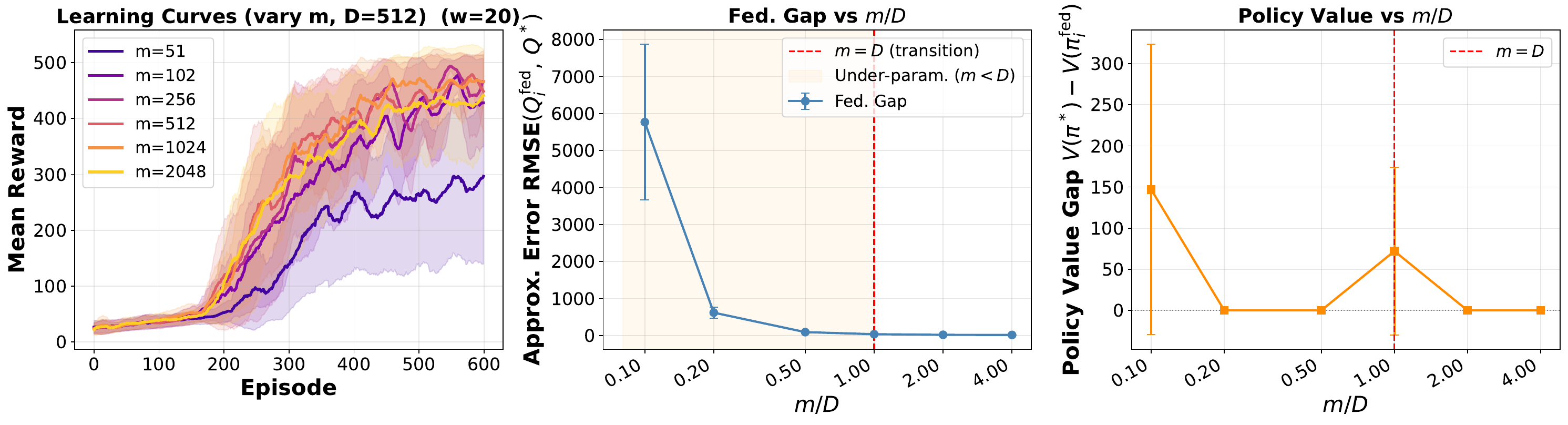}
  \caption{
           Ablation 2 (vary $m$, fix $D=512$): learning curves (left), Q-error vs.\ $m/D$ with the $m=D$ transition (middle), and final policy value (right).}
  \label{fig:ablation2}
\end{figure}

\section{Scalability Analysis}
\label{subsec:scalability}
We evaluate the scalability of FedQHD by varying $N \in \{1, 2, 5, 20\}$ on CartPole and LunarLander, measuring the average reward over the final 30 episodes after 500 episodes. Following standard benchmarks, we define a task as \emph{solved} when the rolling average reward $\geq\!475$ for CartPole and $\geq\!200$ for LunarLander. The dashed red line in each panel marks this threshold; bars that cross it indicate a qualitatively reliable policy.

Figure~\ref{fig:scalability} shows the consistently improved final reward of FedQHD with $N$, yielding gains of $+30.6\%$ on CartPole and $+29.5\%$ on LunarLander relative to $N=1$ with no degradation at large $N$, which demonstrates that federation is a reliable lever for performance in FedQHD. Notably, a modest $N\!=\!5$ already recovers most of the scalability benefit on both environments, suggesting that FedQHD is well-suited to practical deployments where large client counts are often infeasible.

\begin{figure}[ht]
    \centering
    \includegraphics[width=\linewidth]{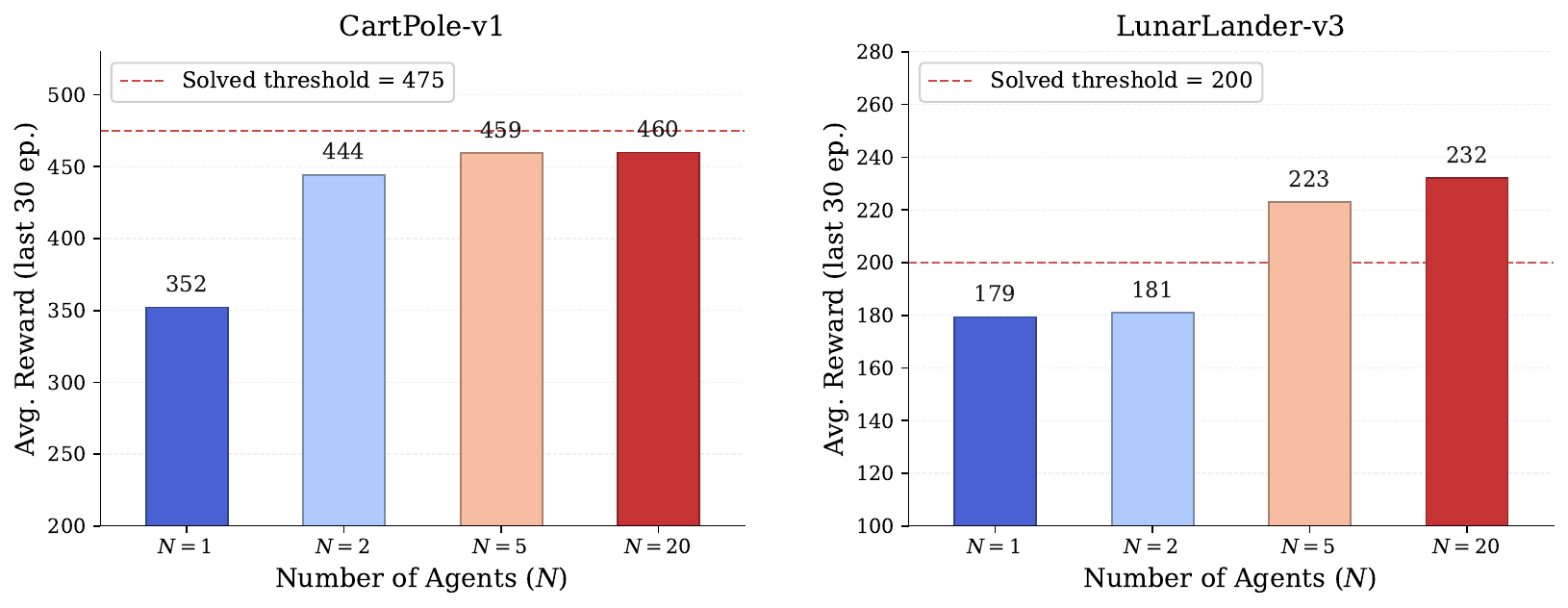}
    \caption{FedQHD performance at 575-600 episodes vs. number of clients $N$.}
    \label{fig:scalability}
\end{figure}

Across homogeneous and heterogeneous settings, FedQHD consistently matches or exceeds federated DQN baselines while requiring substantially less computation. Its performance degradation under encoder mismatch follows the projection regime predicted by Theorem~\ref{thm:fed_gap_main}, and its scalability saturates once the dominant value subspace is sufficiently covered. Together, the results validate the geometric design principles underlying FedQHD.

%% file: appendix/projection_residual_invisibility.tex
\section{Proof of Theorem~\ref{thm:invisibility}}
\label{app:proof_invisibility}
\begin{proof}
By the Woodbury identity applied to $
W^{\mathrm{glob}}_i = \left(X_i^\mathrm{H} X_i + \lambda I_{D_i}\right)^{-1} X_i^\mathrm{H} Q^{\mathrm{glob}}_{\mathrm{ref}}$,
the compiled Q-function evaluates as:
\begin{equation}
    Q_i(s,a;\,W^{\mathrm{glob}}_i)
    = \mathbf{k}_i(s)^{\!\top}
       (G_i+\lambda I_m)^{-1}
       Q^{\mathrm{glob}}_{\mathrm{ref},a}.
\end{equation}
Decompose the teacher:
$Q^{\mathrm{glob}}_{\mathrm{ref},a}
= P_i\,Q^{\mathrm{glob}}_{\mathrm{ref},a} + R_{i,0,a}$,
where
$P_i\,Q^{\mathrm{glob}}_{\mathrm{ref},a} \in \mathrm{col}(X_i)$
and $R_{i,0,a} \in \mathrm{col}(X_i)^{\perp} = \ker(G_i)$.
Since $(G_i + \lambda I_m)^{-1}$ preserves the eigenspaces of $G_i$,
the vector $(G_i+\lambda I_m)^{-1}R_{i,0,a}$ remains in $\ker(G_i)$,
which is orthogonal to
$\mathbf{k}_i(s) \in \mathrm{col}(X_i)$.
Hence $\mathbf{k}_i(s)^{\top}(G_i + \lambda I)^{-1} R_{i,0,a} = 0$. 
\end{proof}

%% file: appendix/encoder_hetero.tex
\section{Proof of Lemma~\ref{lem:term_A}: Encoder Heterogeneity Bias}
\label{app:proof_termA}

We bound the deviation between a client oracle
$\hat Q_i$ and the federation average
$\bar Q = \sum_j \pi_j \hat Q_j$.

\begin{proof}

Recall $\hat Q_k(\cdot,a)=\Pi_k Q^*(\cdot,a)$.
Using the triangle inequality through $Q^*$,
\begin{equation}
\|\hat Q_i - \hat Q_j\|_{\mathcal H}
\le
\|(I-\Pi_i)Q^*\|_{\mathcal H}
+
\|(I-\Pi_j)Q^*\|_{\mathcal H}.
\label{eq:triangle_main}
\end{equation}

We bound the first term (the second follows symmetrically).
Decompose
\[
(I-\Pi_i)Q^*
=
(I-\Pi_i)\Pi_j Q^*
+
(I-\Pi_i)(I-\Pi_j)Q^* .
\]

\textbf{Cross-subspace term.}
Since $\Pi_j Q^*\in\mathcal F_j$,
the definition of principal angle gives
\[
\|(I-\Pi_i)\Pi_j\|_{\mathcal H\to\mathcal H}
=
\sin(\theta_{ij}).
\]
Thus
\[
\|(I-\Pi_i)\Pi_j Q^*\|_{\mathcal H}
\le
\sin(\theta_{ij})
\|\Pi_j Q^*\|_{\mathcal H}
\le
B\sin(\theta_{ij}),
\]
using $\|\Pi_j\|\le1$ and Assumption~\ref{assum:common_rkhs}(i).

\textbf{Residual term.}
Since $\|I-\Pi_i\|\le1$,
\[
\|(I-\Pi_i)(I-\Pi_j)Q^*\|_{\mathcal H}
\le
\|(I-\Pi_j)Q^*\|_{\mathcal H}
\le
\varepsilon_j^{\mathrm{rep}}.
\]

Combining two terms we get $
\|(I-\Pi_i)Q^*\|_{\mathcal H}
\le
B\sin(\theta_{ij})
+
\varepsilon_j^{\mathrm{rep}}$.

By symmetry in $i,j$,
\begin{equation}
    \|\hat Q_i - \hat Q_j\|_{\mathcal H}
\le
2B\sin(\theta_{ij})
+
\varepsilon_i^{\mathrm{rep}}
+
\varepsilon_j^{\mathrm{rep}}.
\label{eq:pairwise_clean}
\end{equation}

Since
$\bar Q
=
\sum_j \pi_j \hat Q_j,
\sum_j \pi_j=1$, by deviating to from the federation average, 
we have
$\hat Q_i - \bar Q
=
\sum_{j\neq i}
\pi_j(\hat Q_i-\hat Q_j).$

Taking the $\mathcal H$-norm and applying the triangle inequality, $
\|\hat Q_i-\bar Q\|_{\mathcal H}
\le
\sum_{j\neq i}
\pi_j
\|\hat Q_i-\hat Q_j\|_{\mathcal H}$

Substituting \eqref{eq:pairwise_clean}, we get 
\begin{equation}
\|\hat Q_i-\bar Q\|_{\mathcal H}
\le
2B\sum_{j\neq i}\pi_j\sin(\theta_{ij})
+
\sum_{j\neq i}\pi_j
(\varepsilon_i^{\mathrm{rep}}
+
\varepsilon_j^{\mathrm{rep}}).
\label{eq:termA_H_final}
\end{equation}

By the reproducing property,
\[
|f(s)|
=
|\langle f,\kappa(\cdot,s)\rangle_{\mathcal H}|
\le
\|f\|_{\mathcal H}\sqrt{\kappa(s,s)}.
\]
Under $\kappa(s,s)\le1$,
$
|\hat Q_i(s,a)-\bar Q(s,a)|
\le
\|\hat Q_i-\bar Q\|_{\mathcal H}.
$

Combining with~\eqref{eq:termA_H_final}
yields the desired bound on Term (A).
\end{proof}

%% file: appendix/regression_error.tex
\section{Proof of Theorem~\ref{thm:fed_gap_main}}
\label{app:proof_termB}
We introduce $\bar{Q}(s,a) \defeq \sum_{i=1}^N \pi_i\hat{Q}_i(s,a)$
as the federation-weighted mean of the oracles, which serves as the global objective defined in~\eqref{eq:qglob_avg}.

\begin{equation}
    |\Delta_i(s,a)|
    \;\leq\;
    \underbrace{|\hat{Q}_i(s,a) - \bar{Q}(s,a)|}_{
    \text{(A): Representation Bias}}
    \;+\;
    \underbrace{|\bar{Q}(s,a)
    - Q_i(s,a;\,W^{\mathrm{glob}}_i(\lambda))|}_{
    \text{(B): Aggregation Distortion}}.
    \label{eq:master_decomp}
\end{equation}
(A) is bound by Lemma~\ref{lem:term_A} and corresponds to Term (I) in the main theorem.

(B) captures how accurately FedQHD's anchor-set ridge
regression recovers the mean oracle. 


We insert an intermediate predictor that uses the same ridge map as the global model but fits client $i$'s own oracle anchor labels. Define the oracle ridge solution
\[
\hat W_i^{\mathrm{ridge}} \defeq (X_i^{\mathrm H}X_i+\lambda I)^{-1}X_i^{\mathrm H}\hat Q_i^{\mathrm{ref}},
\qquad
\hat Q_i^{\mathrm{ref}}=\Re(X_i\hat W_i)\in\mathbb R^{m\times|\mathcal A|}.
\]
Then by triangle inequality,
\[
(B)\le |\bar Q(s,a)-\hat Q_i(s,a)| + |\hat Q_i(s,a)-Q_i(s,a;\hat W_i^{\mathrm{ridge}})| + |Q_i(s,a;\hat W_i^{\mathrm{ridge}})-Q_i(s,a;W_i^{\mathrm{glob}}(\lambda))|.
\]

The first term is exactly Term (A) already bounded in Appendix~\ref{app:proof_termA}. Below we bound the remaining two terms and denote them by $(B_1)$ and $(B_2)$.

\subsection*{Step 1: Bound $(B_1)$ (oracle ridge shrinkage)}
Since $Q_i(s,a;\hat W_i^{\mathrm{ridge}})= \Re(\Phi_i(s)^{\mathrm H}\hat  W_i^{\mathrm{ridge}})=\Re\!\Big(\Phi_i(s)^{\mathrm H}(X_i^{\mathrm H}X_i+\lambda I)^{-1}X_i^{\mathrm H}X_i\hat W_{i,a}\Big)$, we obtain
\[
(B_1) = \hat Q_i(s,a)-Q_i(s,a;\hat W_i^{\mathrm{ridge}})=\Re\!\Big(\Phi_i(s)^{\mathrm H}\bigl[I-(X_i^{\mathrm H}X_i+\lambda I)^{-1}X_i^{\mathrm H}X_i\bigr]\hat W_{i}\Big)=\Re\!\Big(\Phi_i(s)^{\mathrm H}\lambda(X_i^{\mathrm H}X_i+\lambda I)^{-1}\hat W_{i}\Big).
\]

 Decompose $\hat W_{i,a}=P_{X_i^{\mathrm H}}\hat W_{i}+P_{X_i^{\mathrm H}}^{\perp}\hat W_{i}$, where $P_{X_i^{\mathrm H}}$ projects onto $\mathrm{col}(X_i^{\mathrm H})=\mathrm{row}(X_i)$. The two pieces behave differently under the operator $\lambda(X_i^{\mathrm H}X_i+\lambda I)^{-1}$:
 
\emph{(a) On $\mathrm{col}(X_i^{\mathrm H})$:} the operator $X_i^{\mathrm H}X_i$ has eigenvalues bounded below by $\lambda_{\min}^{+}(X_i^{\mathrm H}X_i)=\lambda_{\min}^{+}(G_i)\defeq\gamma_i$, so
\[
\bigl\|\lambda(X_i^{\mathrm H}X_i+\lambda I)^{-1}P_{X_i^{\mathrm H}}\bigr\|_{\mathrm op}=\frac{\lambda}{\gamma_i+\lambda}.
\]
 
\emph{(b) On $\mathrm{col}(X_i^{\mathrm H})^{\perp}$:} the operator acts as the identity with factor $\lambda/\lambda=1$, but this component is invisible to the predictor $\Phi_i(s)^{\mathrm H}$ in the following sense. By Theorem~\ref{thm:invisibility} (projection residual invisibility), the compiled predictor depends only on the in-subspace component of any teacher; equivalently, only the component of $\hat W_{i,a}$ in $\mathrm{col}(X_i^{\mathrm H})$ contributes to $\hat Q_i$ and $Q_i(\cdot;\hat W_i^{\mathrm{ridge}})$ at any state, modulo a feature-space residual that is absorbed into $\varepsilon_i^{\mathrm{rep}}$.
 
Restricting attention to the effective parameters $\hat W_{i,a}^{\parallel}\defeq P_{X_i^{\mathrm H}}\hat W_{i,a}$ and using $\|\Phi_i(s)\|_2\le1$ (Assumption~\ref{assum:common_rkhs}(iii)),
\[
(B_1)\le \|\Phi_i(s)\|_2\,\bigl\|\lambda(X_i^{\mathrm H}X_i+\lambda I)^{-1}P_{X_i^{\mathrm H}}\bigr\|_{\mathrm op}\,\|\hat W_{i,a}^{\parallel}\|_2 \le \frac{\lambda}{\gamma_i+\lambda}\,\|\hat W_i\|_F.
\]


\subsection*{Step 2: Bound $(B_2)$ (teacher mismatch on anchors)}
Now $(B_2)$ measures how ridge fitting reacts when we change the training targets from $\hat Q_i^{\mathrm{ref}}$ to $Q_{\mathrm{ref}}^{\mathrm{glob}}=\sum_j\pi_j\hat Q_j^{\mathrm{ref}}$. Define the anchor-level mismatch for action $a$: $
\Delta^{\mathrm{ref}}_{i,a}\defeq Q^{\mathrm{glob}}_{\mathrm{ref},a}-\hat Q^{\mathrm{ref}}_{i,a}\in\mathbb R^m.
$
Because ridge is linear in the targets, the parameter difference is
\[
W_i^{\mathrm{glob}}(\lambda)-\hat W_i^{\mathrm{ridge}}=(X_i^{\mathrm H}X_i+\lambda I)^{-1}X_i^{\mathrm H}\Delta_i^{\mathrm{ref}}.
\]
Evaluating at $s$ and using the Woodbury form (equivalently the representer form for linear ridge),
\[
Q_i(s,a;W_i^{\mathrm{glob}})-Q_i(s,a;\hat W_i^{\mathrm{ridge}})=\mathbf k_i(s)^\top (G_i+\lambda I)^{-1}\Delta^{\mathrm{ref}}_{i,a}.
\]
Only the component of $\Delta^{\mathrm{ref}}_{i,a}$ inside $\mathrm{col}(X_i)$ matters because $\mathbf k_i(s)\in\mathrm{col}(X_i)$; therefore we may replace $\Delta^{\mathrm{ref}}_{i,a}$ by $P_i\Delta^{\mathrm{ref}}_{i,a}$:
\[
\mathbf k_i(s)^\top (G_i+\lambda I)^{-1}\Delta^{\mathrm{ref}}_{i,a}=\mathbf k_i(s)^\top (G_i+\lambda I)^{-1}P_i\Delta^{\mathrm{ref}}_{i,a}.
\]
To bound this scalar, apply Cauchy--Schwarz after inserting $(G_i)^{\dagger/2}(G_i)^{1/2}$:
\[
\bigl|\mathbf k_i(s)^\top (G_i+\lambda I)^{-1}P_i\Delta^{\mathrm{ref}}_{i,a}\bigr|
\le \sqrt{\mathbf k_i(s)^\top G_i^\dagger\mathbf k_i(s)}\;\sqrt{(P_i\Delta^{\mathrm{ref}}_{i,a})^\top G_i(G_i+\lambda I)^{-2}(P_i\Delta^{\mathrm{ref}}_{i,a})}.
\]
The first factor satisfies $\mathbf k_i(s)^\top G_i^\dagger\mathbf k_i(s)=\|P_{X_i^{\mathrm H}}\Phi_i(s)\|_2^2\le\|\Phi_i(s)\|_2^2\le1$. 
For the second factor, on $\mathrm{col}(X_i)$ the eigenvalues of $G_i$ are bounded below by $\gamma_i=\lambda_{\min}^{+}(G_i)$. Hence $
\max_{t\ge \gamma_i}\frac{t}{(t+\lambda)^2}\le\frac{1}{\gamma_i+\lambda},$. This gives
\[
(P_i\Delta^{\mathrm{ref}}_{i,a})^\top G_i(G_i+\lambda I)^{-2}(P_i\Delta^{\mathrm{ref}}_{i,a})\le \frac{1}{\gamma_i+\lambda}\|P_i\Delta^{\mathrm{ref}}_{i,a}\|_2^2.
\]
Combining yields the bound
\[
(B_2)\le \frac{1}{\sqrt{\gamma_i+\lambda}}\,\|P_i\Delta^{\mathrm{ref}}_{i,a}\|_2.
\]

\subsection*{Step 3: Bound the projected mismatch $\|P_i\Delta^{\mathrm{ref}}_{i,a}\|_2$}
We first expand $\Delta^{\mathrm{ref}}_{i,a}$ using $Q^{\mathrm{glob}}_{\mathrm{ref},a}=\sum_j\pi_j\hat Q^{\mathrm{ref}}_{j,a}$:
\[
\Delta^{\mathrm{ref}}_{i,a}=\sum_{j\ne i}\pi_j(\hat Q^{\mathrm{ref}}_{j,a}-\hat Q^{\mathrm{ref}}_{i,a}),
\qquad
\|P_i\Delta^{\mathrm{ref}}_{i,a}\|_2\le \|\Delta^{\mathrm{ref}}_{i,a}\|_2\le \sum_{j\ne i}\pi_j\|\hat Q^{\mathrm{ref}}_{j,a}-\hat Q^{\mathrm{ref}}_{i,a}\|_2.
\]
Each vector difference collects evaluations of the RKHS function $h_{ij,a}(s)\defeq \hat Q_j(s,a)-\hat Q_i(s,a)$ on the $m$ anchors, so by the reproducing property bound $|h_{ij,a}(s_\ell)|\le \|h_{ij,a}\|_{\mathcal H}$ (using $\kappa(s_\ell,s_\ell)\le1$),
\[
\|\hat Q^{\mathrm{ref}}_{j,a}-\hat Q^{\mathrm{ref}}_{i,a}\|_2 \le \sqrt m\,\|\hat Q_j(\cdot,a)-\hat Q_i(\cdot,a)\|_{\mathcal H}.
\]

Plugging in the pairwise oracle gap bound from Appendix~\ref{app:proof_termA},
\[
\|\hat Q_j(\cdot,a)-\hat Q_i(\cdot,a)\|_{\mathcal H}\le 2B\sin(\theta_{ij})+\varepsilon_i^{\mathrm{rep}}+\varepsilon_j^{\mathrm{rep}},
\]
we obtain
\[
\|P_i\Delta^{\mathrm{ref}}_{i,a}\|_2 \le \sqrt m\sum_{j\ne i}\pi_j\bigl(2B\sin(\theta_{ij})+\varepsilon_i^{\mathrm{rep}}+\varepsilon_j^{\mathrm{rep}}\bigr)=\sqrt m\,\bar h_i.
\]

Combining Steps 1--3 with the Term (A) bound from Lemma~\ref{lem:term_A} (which gives $|\hat Q_i(s,a)-\bar Q(s,a)|\le\bar h_i$), we have
\[
|\Delta_i(s,a)|\le \underbrace{\bar h_i}_{(A)=\text{Term (I)}} + \underbrace{\frac{\sqrt m}{\sqrt{\gamma_i+\lambda}}\bar h_i}_{(B_2)=\text{Term (II)}} + \underbrace{\frac{\lambda}{\gamma_i+\lambda}\|\hat W_i\|_F}_{(B_1)=\text{Term (III)}},
\]

\qed